\documentclass[11pt]{article}

\usepackage{algorithm}
\usepackage{algorithmic}

\usepackage{natbib}
\usepackage{amsthm}

\newtheorem{definition}{Definition}
\newtheorem{lemma}{Lemma}
\newtheorem{proposition}{Proposition}
\usepackage{amssymb}
\usepackage{booktabs}
\usepackage[preprint]{acl}
\usepackage{threeparttable}
\usepackage{times}
\usepackage{latexsym}

\usepackage[T1]{fontenc}
\usepackage[utf8]{inputenc}

\usepackage{microtype}
\usepackage{amsthm}

\usepackage{amsmath}

\usepackage{inconsolata}
\usepackage{geometry}
\usepackage{amsmath}

\usepackage{multicol}
\usepackage{graphicx}

\usepackage{algorithmic}

\title{SteinerSQL: Graph-Guided Mathematical Reasoning \\ for Text-to-SQL Generation}

\author{
  Xutao Mao \\
  Vanderbilt University \\
  \texttt{xutao.mao@vanderbilt.edu}
  \And
  Tao Liu \\
  Zhengzhou University \\
  \texttt{taoliu01@gs.zzu.edu.cn}
  \And
  Hongying Zan \\
  Zhengzhou University \\
  \texttt{iehyzan@zzu.edu.cn}
}
\usepackage{amsthm}

\begin{document}
 \maketitle

\begin{abstract}
Large Language Models (LLMs) struggle with complex Text-to-SQL queries that demand both sophisticated mathematical reasoning and intricate schema navigation. Existing methods often tackle these challenges in isolation, creating a fractured reasoning process that compromises logical and structural correctness. To resolve this, we introduce SteinerSQL, a framework that unifies these dual challenges into a single, graph-centric optimization problem. SteinerSQL operates in three stages: mathematical decomposition to identify required tables (terminals), optimal reasoning scaffold construction via a Steiner tree problem, and multi-level validation to ensure correctness. On the challenging LogicCat and Spider2.0-Lite benchmarks, SteinerSQL establishes a new state-of-the-art with 36.10\% and 40.04\% execution accuracy, respectively, using Gemini-2.5-Pro. Beyond accuracy, SteinerSQL presents a new, unified paradigm for Text-to-SQL, paving the way for more robust and principled solutions to complex reasoning tasks.
\end{abstract}

\section{Introduction}
Text-to-SQL generation converts natural-language questions into executable SQL, promising frictionless access to structured data \cite{yu2018spider,10.1007/s00778-022-00776-8,DBLP:journals/corr/abs-2406-08426}. While large language models (LLMs) achieve impressive performance on standard benchmarks \cite{10.14778/3641204.3641221,pourreza2025chasesql}, they face significant challenges when confronted with complex mathematical reasoning that requires multi-step analytical computations across interconnected database relationships \cite{lei2025spider,deng2025reforce}. Although recent advances suggest that modern LLMs can handle schema navigation more effectively when schema fit within context windows \cite{sun2023sql,maamari2024the,chung2025long}, mathematical reasoning introduces different complexities that require structured approaches to preserve computational dependencies across schema relationships \cite{zheng-etal-2024-archer,he2025star,stoisser2025sparks}.

The core challenge lies in mathematical reasoning complexity within SQL generation \cite{zheng-etal-2024-archer,he2025star}. Questions involving statistical analysis, multi-level aggregations, conditional mathematics, and cross-table numerical dependencies create reasoning spaces that overwhelm current LLM approaches \cite{guo2024evaluating,mitsopoulou2025analysis,yun2025seed}. Specifically, current LLM-based approaches struggle with such mathematical reasoning for two reasons: (1) \textbf{Mathematical decomposition complexity}: LLMs lack systematic frameworks for breaking down complex mathematical operations into coherent SQL analytical workflows \cite{tai2023exploring,guo2024evaluating,he2025star}, often generating queries that compute incorrect intermediate results or miss essential mathematical relationships \cite{10.14778/3641204.3641221,mitsopoulou2025analysis,zhai2025excot}. (2) \textbf{Schema navigation under mathematical constraints}: When mathematical operations require data from multiple tables, LLMs struggle to identify the optimal sufficient schema connections needed to support the computational logic \cite{he2025star,shen2025study}, leading to either incomplete reasoning paths that break mathematical dependencies or excessive joins that introduce computational noise.

These challenges naturally map to a graph optimization problem.
Prior work has shown that mathematical reasoning questions in
Text-to-SQL correspond to finding optimal sufficient connections
that preserve computational dependencies in the database schema
graph \citep{baik2019bridging,nascimento2025text}.
This correspondence emerges because mathematical operations
inherently require specific data relationships to be maintained—aggregations need their constituent elements \citep{fent2021practical}
and calculations demand complete computational chains from input to output \citep{weng2025graph,shen2025study}.
 \begin{figure*}[h]
    \centering
    \includegraphics[width=0.8\linewidth]{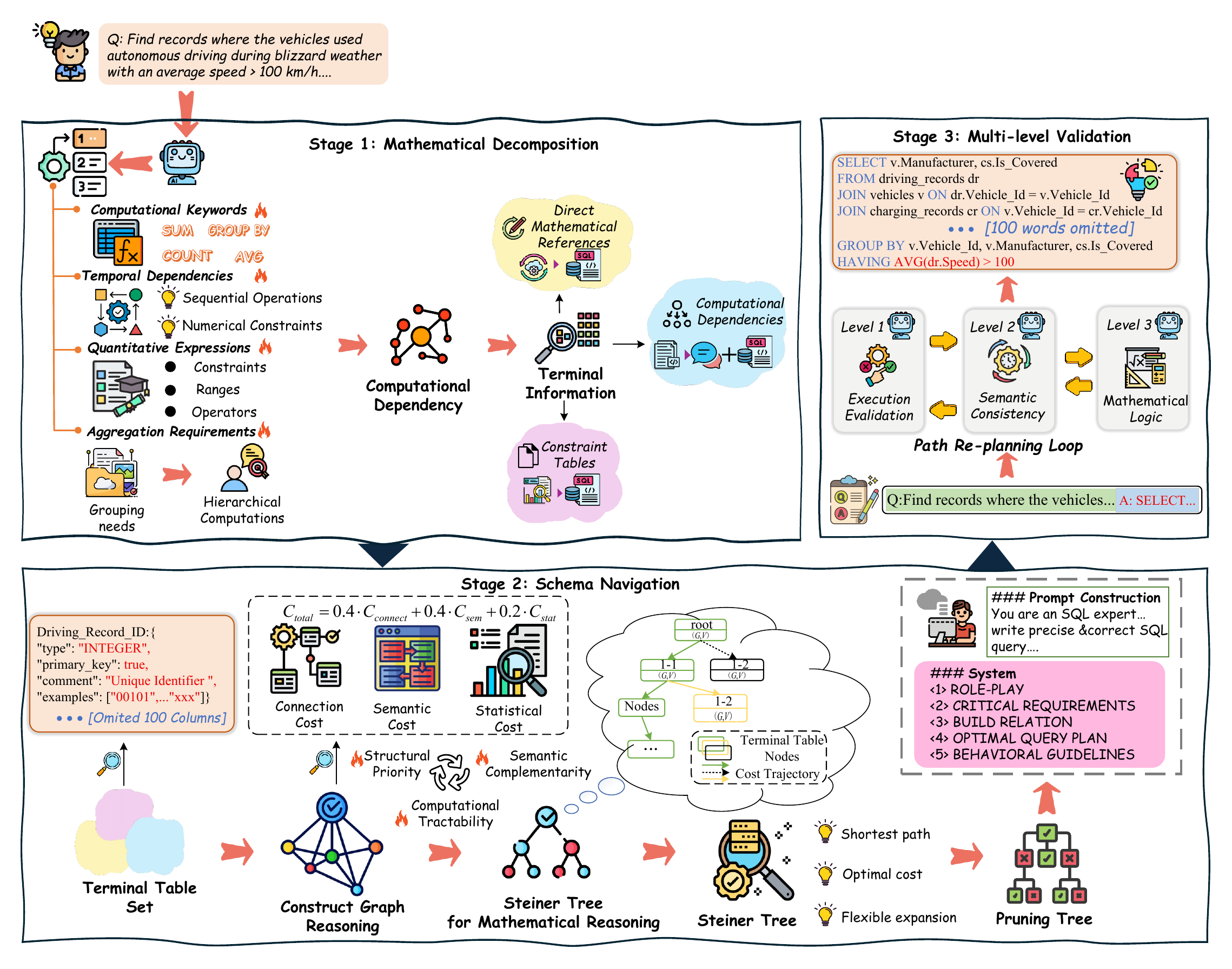}
    \caption{ An overview of the SteinerSQL framework. The workflow consists of three main stages: (1) Mathematical Decomposition, which deconstructs the user's question into formal computational requirements and identifies terminal tables; (2) Schema Navigation, which models the problem as a Steiner tree on the schema graph to find an optimal reasoning scaffold; and (3) Multi-level Validation, which verifies the generated SQL for correctness and triggers a re-planning loop if errors are detected.}
    \label{fig:framework_overview}
\end{figure*}

\newcommand{\cmark}{\ding{51}} % ✓
\newcommand{\xmark}{\ding{55}} % ✗
\newcommand{\tmark}{\ding{115}}% △ (partial support)
% -------------------------------------------------------

This paper makes the following key contributions:
\begin{itemize}
    \item We propose a novel approach that integrates mathematical decomposition, graph-based schema navigation, and multi-level validation into a cohesive pipeline. This method bridges the gap between logical reasoning and structural traversal by formulating the challenge as a Steiner tree optimization problem.
    \item Our framework establishes a new state-of-the-art on two challenging benchmarks, achieving 36.10\% execution accuracy on LogicCat and 40.04\% on Spider2.0-Lite. 
\end{itemize}

\section{Related Work}

\noindent \textbf{Structured Reasoning Gaps in Text-to-SQL.} We state that LLM-based Text-to-SQL models face two critical challenges: mathematical decomposition complexity and schema navigation under mathematical constraints. Recent approaches address each challenge separately. For mathematical reasoning, STaR-SQL achieves 31.6\% improvement through self-taught reasoning~\citep{he2025star}, while ExCoT improves execution accuracy from 57.37\% to 68.51\% on BIRD via execution-guided optimization~\citep{zhai2025excot}. For schema navigation, graph-based methods like TEMPLAR formulate join-path inference as Steiner tree search~\citep{baik2019bridging}, while Nascimento et al.~\citep{nascimento2025text} retrieves minimal Steiner subgraphs to constrain generation. These approaches solve challenges in isolation—mathematical reasoning improvements ignore schema structure, while graph-based navigation methods delegate arithmetic operations entirely to LLMs, making the system unable to solve complex problems that require both an understanding of data structure and precise mathematical calculations.

\section{The SteinerSQL Framework}

% REFINED: The introduction now clearly states the core idea of "bridging" the gap between logic and structure, setting the stage for the three-stage process.
Our framework, \textbf{SteinerSQL}, builds a principled bridge between the abstract mathematical logic of a question and the concrete structure of a database. We achieve this by casting the discovery of a valid reasoning path as a Steiner tree problem~\cite{baik2019bridging,nascimento2025text} on the schema graph. The goal is to find a \textit{reasoning scaffold}---the lowest-cost subgraph that links all mathematically required tables (terminals) while preserving the full computational flow. This integrated approach is implemented as a three-stage pipeline: (i)~\textbf{Mathematical Decomposition}, (ii)~\textbf{Schema Navigation}, and (iii)~\textbf{Multi-Level Validation}. We demonstrated two case examples how this pipeline is executed in Appendix \ref{append:case}. 

\subsection{Overall Pipeline}
In the SteinerSQL framework, the LLM acts as the core reasoning agent, guided by a structured system prompt organized into five components: Role-Play, Critical Requirements, Build Relation, Optimal Query Plan, and Behavioral Guidelines shown in Figure \ref{fig:framework_overview}. This design creates a synergy between the LLM's language capabilities and the framework's deterministic graph algorithms. Initially, the LLM deconstructs the user's query to identify the required terminal tables. The framework's algorithmic core then takes these terminals, constructs a schema graph, and solves the Steiner tree problem to produce an optimal reasoning scaffold. This scaffold is subsequently passed back to the LLM, which translates the structured path into the final SQL query and performs a multi-level validation check to ensure correctness. The psuedocode of the LLM agent pipeline for this framework is shown in Algorithm \ref{algo:steiner_sql} in Appendix \ref{sec:appendixmethod}.

\subsection{Stage 1: Mathematical Decomposition}
This initial stage deconstructs the natural language question to identify all mathematical entities and their computational dependencies. This process establishes the formal requirements that will guide the construction of the reasoning scaffold.

\paragraph{Mathematical Entity Extraction.} We systematically extract mathematical entities from natural language questions through multi-dimensional analysis: (1) Computational Keywords: Identify mathematical operations (SUM, COUNT, AVG, etc.) and their target attributes; (2) Quantitative Expressions: Extract numerical constraints, ranges, and comparison operators; (3) Aggregation Requirements: Detect grouping needs and hierarchical computations; (4) Temporal Dependencies: Identify time-based calculations and sequential operations.

\paragraph{Computational Dependency Analysis.}
After extracting mathematical entities, we analyze their computational dependencies to understand the logical flow through three key steps shown in Algorithm \ref{algo:math}.

\paragraph{Terminal Table Identification.} Based on the mathematical entity and dependency analysis, we identify the essential tables, or terminals, that are required to fulfill the query's mathematical logic: (1) Direct Mathematical References: Tables containing attributes mentioned in mathematical expressions; (2) Computational Dependencies: Tables required for intermediate calculations; (3) Constraint Tables: Tables needed to satisfy mathematical constraints and filters.

Thus, this stage produces a formal set of terminal tables, denoted as $\mathcal{T}_{\text{req}}$, which represent the indispensable schema components for the query's mathematical logic. This set $\mathcal{T}_{\text{req}}$ serves as the direct input---the set of terminals---for the Steiner tree problem we will formally define and solve in the subsequent Schema Navigation stage.

\subsection{Stage 2: Schema Navigation}
This stage takes the set of terminal tables $\mathcal{T}_{\text{req}}$ from Stage 1 and navigates the database schema to construct an optimal reasoning scaffold. To achieve this in a principled manner, we first formally define the search space as a schema graph and then formulate our navigation task as a Steiner tree problem on this graph.

\subsubsection{Definitions}
\label{sec:schema}
\textbf{Definition 1 (Schema Graph).} Given a database schema $\mathcal{S} = (\mathcal{T}, \mathcal{R})$, we define the schema graph as a weighted undirected graph $G_{\mathcal{S}} = (V, E, C)$, where: (1) The set of vertices $V$ corresponds to the set of tables $\mathcal{T}$. (2) An edge $(t_i, t_j) \in E$ exists if a set of join predicates is met: a foreign key relationship exists between $t_i$ and $t_j$, or their highest column name/type similarity exceeds a reproducible threshold $\tau$. (3) $C$ is a cost function assigning a positive weight $C(e)$ to each edge $e \in E$, representing join dis-utility. Since $G_{\mathcal S}$ is finite and $w(e)\ge 0$, such a minimizer exists.

\textbf{Definition 2 (Steiner Tree on a Schema Graph).}
Given $G_{\mathcal S}=(V,E,C)$ and terminals $\mathcal T_{\text{req}}\subseteq V$, a Steiner tree is a connected, acyclic subgraph $T\subseteq G_{\mathcal S}$ that spans $\mathcal T_{\text{req}}$ and minimizes
\[
C_{\text{total}}(T)=\sum_{e\in E(T)}C(e),
\]
where $C(e)\ge 0$ and $C_{\text{total}}$ is additive.

\subsubsection{Modeling Reasoning Scaffolds}
\label{sec:modeling}
\noindent\textbf{Assumption.}
We assume (i) schema fidelity, (ii) terminal soundness, (iii) additive and monotone nonnegative edge costs, and (iv) operation locality.
See Appendix~\ref{app:steiner-proof} (A1--A4) for precise statements.
.

\begin{algorithm}[t]
\caption{Mathematical Dependency Analysis}
\label{algo:math}
\begin{algorithmic}
\STATE \textbf{Input:} Question $q$, Mathematical entities $\mathcal{E}_{\text{math}}$
\STATE \textbf{Output:} Dependency graph $G_{\text{dep}}$, Terminal set $\mathcal{T}_{\text{req}}$
\STATE Initialize: $G_{\text{dep}} \leftarrow$ empty graph, $\mathcal{T}_{\text{req}} \leftarrow \emptyset$

\STATE \textbf{// Step 1: Data Flow Analysis}
\FOR{each entity $e \in \mathcal{E}_{\text{math}}$}
    \STATE $attr_{\text{req}} \leftarrow$ \textsc{ExtractRequiredAttributes}($e$)
    \STATE $tables_{\text{direct}} \leftarrow$ \textsc{FindContainingTables}($attr_{\text{req}}$)
    \STATE $\mathcal{T}_{\text{req}} \leftarrow \mathcal{T}_{\text{req}} \cup tables_{\text{direct}}$
    \STATE Add dependency edges for attribute flow to $G_{\text{dep}}$
\ENDFOR

\STATE \textbf{// Step 2: Join Requirements Analysis}
\FOR{each pair $(t_i, t_j) \in \mathcal{T}_{\text{req}}$}
    \IF{\textsc{RequiresJoin}($t_i, t_j$) based on mathematical operations}
        \STATE $join_{\text{path}} \leftarrow$ \textsc{FindJoinPath}($t_i, t_j$)
        \STATE Add $join_{\text{path}}$ tables to $\mathcal{T}_{\text{req}}$
        \STATE Add join dependencies to $G_{\text{dep}}$
    \ENDIF
\ENDFOR

\STATE \textbf{// Step 3: Constraint Propagation}
\FOR{each constraint $c$ in mathematical entities}
    \STATE $affected_{\text{tables}} \leftarrow$ \textsc{PropagateConstraint}($c, \mathcal{T}_{\text{req}}$)
    \STATE Update $\mathcal{T}_{\text{req}}$ with constraint-required tables
    \STATE Add constraint edges to $G_{\text{dep}}$
\ENDFOR

\STATE \textbf{return} $G_{\text{dep}}$, $\mathcal{T}_{\text{req}}$
\end{algorithmic}
\end{algorithm}

\noindent\textbf{Theorem.} \textit{Under Assumption, the minimum Steiner tree on $G_{\mathcal{S}}$ with terminals $\mathcal{T}_{\text{req}}$ corresponds to the optimal (most efficient and relevant) reasoning scaffold required for the mathematical query.}

\noindent\textbf{Proof Sketch.} The mapping is direct: (1) \textbf{Terminals} $\mathcal{T}_{\text{req}}$ are the tables essential for the query's mathematical logic. (2) \textbf{Connectivity} in the Steiner tree ensures a valid join path exists for all required tables, a prerequisite for a valid SQL query. (3) \textbf{Cost Minimization} finds the closest set of intermediate tables and joins, minimizing complexity.

\subsubsection{Edge Cost Formulation}
\label{sec:cost}
We construct the schema graph $G_{\mathcal{S}}$ using a holistic cost function where lower values indicate better joins. To ensure consistency, all similarity and correlation metrics used in the cost components are normalized to the range $[0, 1]$. The total cost is an aggregation of structural, semantic, and statistical dissimilarity\footnote{We reproducibly add similarity edges using a frozen column-name encoder (\texttt{all-MiniLM-L6-v2}) and a rule-based type match. For tables $(t_i,t_j)$, compute
$s=\max_{c_i\in t_i,\,c_j\in t_j}\big[\alpha\,\cos(e(c_i),e(c_j))+(1-\alpha)\,\mathbf{1}\{\text{type}(c_i)=\text{type}(c_j)\}\big]$
with $\alpha=0.85$, and add $(t_i,t_j)$ iff $s\ge\tau$ with $\tau=0.75$.}:
\begin{equation}
C_{\text{total}} = \alpha \cdot C_{\text{connect}} + \beta \cdot C_{\text{sem}} + \gamma \cdot C_{\text{stat}}.
\label{eq:weight_formula}
\end{equation}
The cost components are:
 \begin{itemize}
 \item \textbf{Connection Cost ($C_{\text{connect}}$):} Measures structural distance based on foreign keys \cite{guo-etal-2019-towards,li2021putting}, name dissimilarity \cite{gan-etal-2023-appraising,zhang-etal-2023-nameguess}, and type incompatibility \cite{bernstein2011generic}. The table-pair similarities $\text{sim}_{\text{name}}$ and $\text{sim}_{\text{type}}$ are calculated by taking the maximum similarity score over column pairs between the tables.
 \begin{equation}
 \begin{aligned}
 C_{\text{connect}} &= w_1 \cdot \mathbb{I}_{\neg\text{FK}}(t_i, t_j) \\ & + w_2 \cdot (1 - \text{sim}_{\text{name}}(t_i, t_j)) \\
 &\quad + w_3 \cdot (1 - \text{sim}_{\text{type}}(t_i, t_j)).
 \end{aligned}
 \end{equation}
 where $\mathbb{I}_{\neg\text{FK}}$ is an indicator function that is 0 if a direct FK exists and 1 otherwise. We set the internal weights equally.

 \item \textbf{Semantic Cost ($C_{\text{sem}}$):} Captures semantic distance using SentenceTransformer embeddings (all-MiniLM-L6-v2) \cite{reimers2019sentence}. A table's embedding vector $\mathbf{e}_{t}$ is derived by averaging the embeddings of its table name and all of its column names.
 \begin{equation}
 C_{\text{sem}} = 1 - \cos(\mathbf{e}_{t_i}, \mathbf{e}_{t_j}).
 \end{equation}

 \item \textbf{Statistical Cost ($C_{\text{stat}}$):} Measures statistical implausibility from join selectivity \cite{wu2023factorjoin} and correlation \cite{zhu2023nocap,halford2020selectivity}.
\begin{equation}
\begin{aligned}
 C_{\text{stat}} &= w_4 \cdot (1 - \text{join\_selectivity}(t_i, t_j)) \\
 &\quad + w_5 \cdot (1 - \text{corr\_strength}(t_i, t_j)).
\end{aligned}
\end{equation}
The internal weights are set equally with $w_4=0.5$ and $w_5=0.5$.
 \end{itemize}

The weight selection for our multi-dimensional cost function is guided by the following theoretical principles: (1) Connection costs is important ($\alpha = \beta > \gamma$) because database integrity constraints provide the fundamental feasibility boundaries for any valid query \cite{codd1970relational,aho1979theory}; (2) Semantic costs should also exceed statistical costs ($\beta > \gamma$) as mathematical intent understanding is more critical than statistical optimization for correctness \cite{wang-etal-2022-improving-text,shen2025study}.

Based on these principles and empirical validation in Section \ref{sec:ablation}, we derive our choice: $\alpha = 0.4, \quad \beta = 0.4, \quad \gamma = 0.2$. This allocation reflects the hierarchical importance: structural feasibility, semantic alignment to statistical optimization.
\subsubsection{Steiner Tree Solution}
We use the Kou–Markowsky–Berman (KMB) 2-approximation algorithm~\cite{kou1981fast} to solve the Steiner Tree problem, which requires the graph distances to satisfy the triangle inequality. The solution process is as follows:
\begin{enumerate}
    \item \textbf{Metric Closure:} First, we compute the metric closure of $G_{\mathcal{S}}$ by running an all-pairs shortest path algorithm (Floyd-Warshall). This creates a new complete graph $G'_{\mathcal{S}}$ where the edge weight between any two nodes is the shortest path distance in $G_{\mathcal{S}}$.
    \item \textbf{MST on Terminals:} We construct a minimal spanning tree (MST) on the subgraph of $G'_{\mathcal{S}}$ induced by the terminal set $\mathcal{T}_{\text{req}}$.
    \item \textbf{Scaffold Construction:} The edges of the MST are mapped back to their corresponding shortest paths in the original graph $G_{\mathcal{S}}$. The union of these paths forms a subgraph that connects all terminals.
    \item \textbf{Pruning:} A final traversal removes any cycles to yield the final Steiner tree.
\end{enumerate}
This principled approach guarantees the resulting scaffold's cost is at most twice that of the optimal solution \cite{kou1981fast}. The time complexity of this approach is dominated by the Floyd-Warshall algorithm, resulting in $O(|V|^3)$, where $|V|$ is the number of tables in the schema. In cases where multiple paths or edges have identical costs during the shortest path or MST construction, we employ a deterministic tie-breaking rule based on the lexicographical order of the table and column names involved, ensuring reproducibility.
\subsection{Stage 3: Multi-level Validation}
To ensure final correctness, we employ a three-level validation process with precise checks designed to catch errors common in complex SQL queries. If a Level 2 or 3 error is detected, which indicates a flawed reasoning scaffold, our Path Re-planning Loop is triggered to generate a new scaffold with updated constraints.

\paragraph{Level 1: Execution Validation.} This foundational check verifies that the generated SQL is syntactically correct and executable against the database schema. It primarily catches surface-level errors such as invalid SQL syntax, or incorrect table and column names that would cause immediate database execution failures.

\paragraph{Level 2: Semantic Consistency Validation.} This level scrutinizes the query's logical alignment with the user's intent. Specific checks include verifying that all terminal tables identified in Stage 1 are present in the query's \texttt{FROM} or \texttt{JOIN} clauses, ensuring that join conditions use semantically appropriate columns and that the scaffold does not contain joins irrelevant to the query, and confirming that attributes in the \texttt{SELECT}, \texttt{WHERE}, and \texttt{HAVING} clauses correctly map to the entities mentioned in the natural language question.

\paragraph{Level 3: Mathematical Logic Validation.} This is the deepest level, auditing the computational structure of the query. The checks are highly specific: (1) enforcing the rule that any non-aggregated column in the \texttt{SELECT} list also appears in the \texttt{GROUP BY} clause; (2) validating that the aggregation functions used (e.g., \texttt{AVG}, \texttt{SUM}, \texttt{COUNT}) and the columns they apply to directly match the operations requested in the prompt; (3) Checking that numerical constraints are correctly translated into \texttt{WHERE} or \texttt{HAVING} clauses with proper operators and values.

\paragraph{Path Re-planning Loop.} When a semantic (Level 2) or mathematical (Level 3) error is detected, the framework treats it as a signal that the underlying reasoning scaffold is flawed. The validation error is translated into a new constraint for the graph search in Stage 2. 
\section{Experiments}
 \subsection{Experimental Setup}
 \subsubsection{Datasets}
We evaluate our work on four challenging Text-to-SQL benchmarks which are all avaliable in their paper \footnote{For the LogicCat benchmark, our experiments utilize the complete dataset, which is an extension of the publicly available version and is accessible upon request from the original authors.}: (1) \textbf{LogicCat \cite{2505.18744}:} a challenging text-to-SQL benchmark comprising 4,038 English questions focusing on complex multi-domain reasoning including physical knowledge, mathematical logic, commonsense reasoning, and hypothetical scenarios in real-world setting. We use the official public subset includes 2,369 questions for evaluation. (2) \textbf{Spider \cite{yu2018spider}}: A foundational, large-scale, cross-domain Text-to-SQL benchmark. It consists of 10,181 questions and 200 databases, focusing on complex query structures and testing a model's ability to generalize to new database schemas. We use the 1,034 questions from the test set for evaluation. (3) \textbf{Spider2.0-Lite \cite{lei2025spider}:} Spider 2.0 is a challenging benchmark that evaluates language models on 632 real-world enterprise text-to-SQL workflows. We use the derivative of Spider2.0-Lite which includes 547 questions and host on diverse databases such as Snowflake, BigQuery and SQLite. Our pipeline is adopted into successfully execution for this benchmarking evaluation. (4) \textbf{BIRD \cite{li2023can}:} BIRD is the classic benchmark for Text-to-SQL evaluation, which contains over 12,751 questions and 95 databases. We use the dev subset for evaluation which includes 1,534 questions.

 \subsubsection{Backbone LLMs}
To evaluate the generalization of our framework, we integrate SteinerSQL with four diverse, state-of-the-art LLMs. Our selection is consistent with contemporary literature \cite{pourreza2023dinsql,pourreza2025chasesql,2505.18744} in complex reasoning tasks: DeepSeek-R1-0528 \cite{2501.12948}, GPT-4o-2024-11-20 \cite{openai-4o}, GPT-4.1-2025-04-14 \cite{openai-4.1}, and Gemini-2.5-Pro-preview-05-06 \cite{google-gemini-2-5-pro}.

 \subsubsection{Model Comparison}
 Our method is compared against nine state-of-the-art Text-to-SQL methods across LogicCat, Spider2.0-Lite, and BIRD-dev datasets. These \textbf{published baselines} include DIN-SQL~\cite{pourreza2023dinsql}, DAIL-SQL~\cite{10.14778/3641204.3641221}, CHESS~\cite{2405.16755}, XiYanSQL-QwenCoder-32B~\cite{2411.08599}, AskData~\cite{shkapenyuk2025automatic}, CHASE-SQL~\cite{pourreza2025chasesql}, LinkAlign~\cite{wang2025linkalign}, OmniSQL-32B~\cite{2503.02240}, and ReFoRCE~\cite{deng2025reforce}, representing diverse methodological approaches including decomposed prompting, in-context learning, hierarchical encoding, and reinforcement learning.
 
 We also compare our SteinerSQL approach with \textit{standard prompting} baselines applied to our four backbone LLMs (DeepSeek-R1, GPT-4o, GPT-4.1, and Gemini-2.5-Pro) across all datasets\footnote{For LogicCat evaluation, we adopt the execution accuracy of the standard prompting with four backbones reported in published papers because we use the same prompting strategy. For Spider2.0-Lite specifically, we evaluate using the Spider-agent with our four backbones for standard prompting comparison.}. This standard prompting includes the same template of role-playing, critical requirements, and behavioral guidelines with chain-of-thought annotation (LogicCat), evidence and schema (BIRD and Spider) or external knowledge (Spider2.0-Lite) that the datasets provided, but without our specialized three-stage framework components.

 \subsubsection{Evaluation Metric}
 We adopt to use the widely-used metric: Execution Accuracy (EX) \cite{yu2018spider} as our primary metric. This evaluates where a query is correct only if its execution result matches that of the ground-truth query. For each example, we use the evaluation script to produce a score of either 0 or 1.

 \subsubsection{Implementation Details}
Experiments were conducted on a cloud-based server with two NVIDIA A100 GPUs. LLM inference was managed via the Hugging Face repository and official APIs. We set all LLMs with a temperature of 0, which is in greedy decoding settings. Hyperparameters of all baselines are default associated with their original implementation and all other hyperparameters of LLMs are default. All databases are hosted and managed using SQLite except Spider-Lite and we follow the official documentations for functions and usage.

\begin{table}[h]
\centering
\begin{threeparttable} % <-- Wrapper for table and notes
\setlength{\tabcolsep}{0.4mm}
\begin{tabular}{lccccc}
\toprule
\textbf{Methods} & \textbf{LogicCat} & \textbf{Spider} & \textbf{Spider2.0} & \textbf{BIRD} \\
\midrule
\textit{Baselines} \\
% --- Added superscripts for backbone LLMs ---
DIN-SQL & 19.28†\tnote{a} & 85.30†\tnote{b} & 1.46†\tnote{b} & 50.72†\tnote{a} \\
DAIL-SQL & 23.71†\tnote{a} & 86.17†\tnote{b} & 5.68†\tnote{b} & 54.76†\tnote{a} \\
CHESS & 33.20†‡\tnote{a} & 87.23 & 3.84†\tnote{c} & 68.31†\tnote{a} \\
XiYanSQL & 20.11† & 89.65† & -- & 69.03† \\
MiniSeek & -- & \textbf{91.20†‡} & -- & -- \\
AskData & -- & -- & -- & \textbf{75.36†‡} \\
CHASE & 29.81†\tnote{a} & 87.62†\tnote{c} & -- & 74.90†\tnote{d} \\
LinkAlign & -- & -- & 33.09†‡ & -- \\
OmniSQL & 22.41† & 88.30 & -- & 72.05† \\
ReFoRCE & -- & -- & 30.35† & -- \\
\midrule
\textit{Standard} \\
DS-R1 & 28.17† & 86.00 & 24.50 & 66.69 \\
GPT-4o & 28.03† & 85.49 & 14.99 & 60.10 \\
GPT-4.1 & 30.06† & 85.30 & 13.35 & 64.80 \\
Gemini-2.5 & 29.26† & 86.27 & 26.87 & 70.08 \\
\midrule
\textit{SteinerSQL} \\
DS-R1 & 33.98 & 88.59 & 34.92 & 70.01 \\
GPT-4o & 32.00 & 87.43 & 25.05 & 64.93 \\
GPT-4.1 & 35.04 & 87.23 & 26.33 & 67.08 \\
Gemini-2.5 & \textbf{36.10} & 88.30 & \textbf{40.04} & 73.92 \\
\bottomrule
\end{tabular}
\caption{Main results on Execution Accuracy (EX) in percentages (\%) across datasets. Our method is presented with four different backbone LLMs. † indicates results from published papers/leaderboards. ‡ indicates current state-of-the-art on that dataset. -- indicates results not available.}
\label{tab:main_results}
\begin{tablenotes}[para,flushleft] % <-- Legend for the superscripts
    \small
    \textit{Backbone LLMs for baselines:}
    \item[a] GPT-4o
    \item[b] GPT-4
    \item[c] Gemini-1.5-Pro
    \item[d] Gemini-2.5-Pro
\end{tablenotes}
\end{threeparttable}
\end{table}

 \subsection{Main Results}

\subsubsection{Performance Comparison}
As shown in Table~\ref{tab:main_results}, our method, SteinerSQL, significantly outperforms existing state-of-the-art (SOTA) frameworks and standard prompting baselines across all datasets and backbone LLMs.

On the complex reasoning benchmark LogicCat, SteinerSQL (Gemini-2.5-Pro) establishes a new SOTA with 36.10\% EX. On Spider2.0-Lite, it also achieves a new SOTA of 40.04\% EX, surpassing the previous best by over 3.6\%. Furthermore, our method demonstrates highly competitive performance on BIRD-dev and Spider-test with 73.92\% EX and 88.59\% EX, closely approaching the current SOTA.

 \begin{figure}[h]
     \centering
     \includegraphics[width=\linewidth]{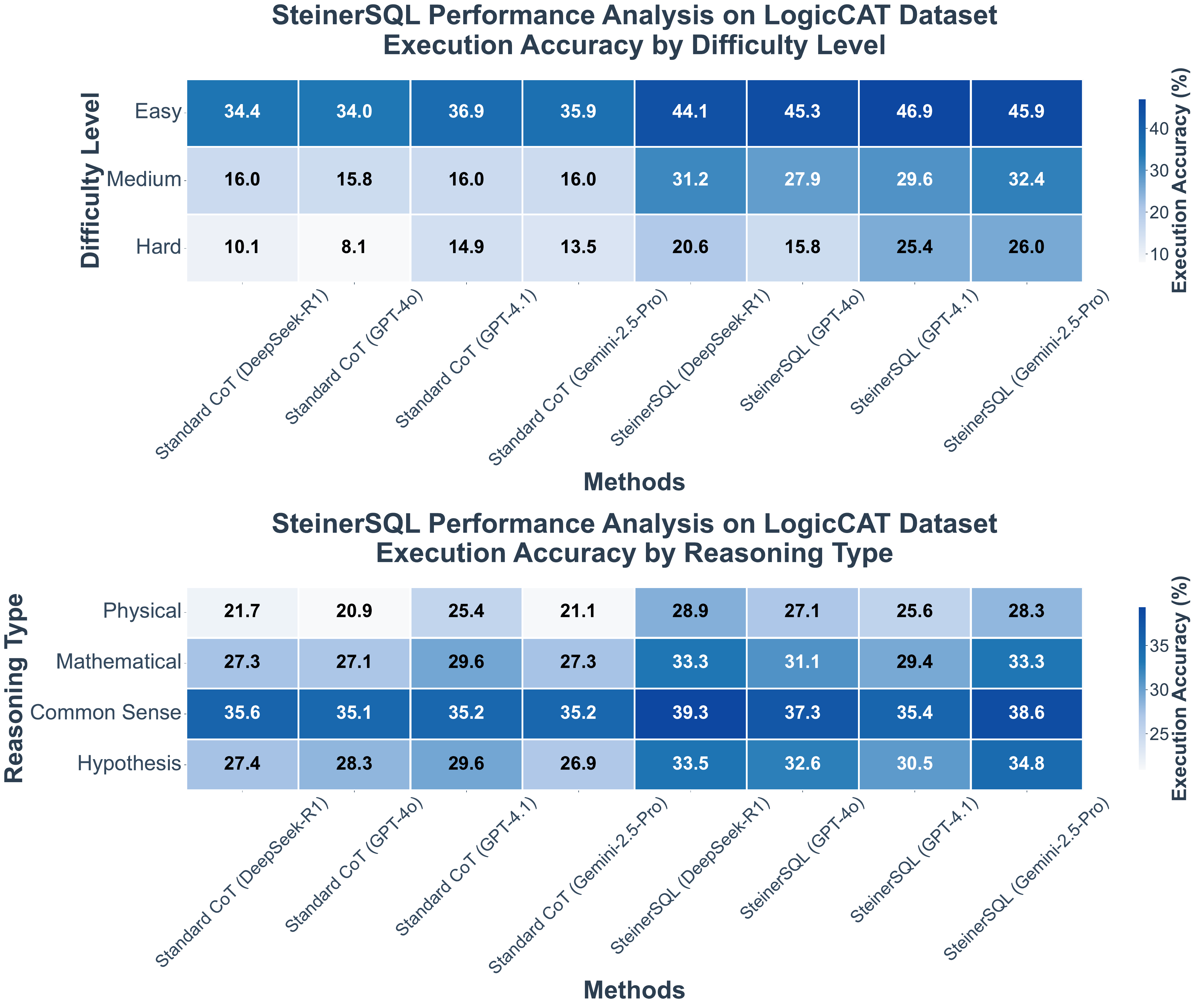}
     \caption{Performance heatmap of different methods on the LogicCat dataset across three difficulty levels (Easy, Medium, Hard) and different reasoning types. The heatmap reveals that SteinerSQL methods consistently outperform all standard prompting baselines across difficulty levels and reasoning types.}
     \label{fig:LogicCat_heatmap}
 \end{figure}  
Figure~\ref{fig:LogicCat_heatmap} details our performance on LogicCat, where the advantage of SteinerSQL is most pronounced on the Hard difficulty subset. On these challenging queries, our methods achieve accuracies up to 26.0\%, substantially outperforming standard prompting baselines. This highlights our workflow's superior ability to handle complex queries. When analyzed by reasoning category, SteinerSQL shows significant gains in tasks requiring mathematical and hypothesis-based reasoning, improving performance by up to 6 absolute points over standard methods.

 \begin{figure}[htbp]
     \centering
     \includegraphics[width=\linewidth]{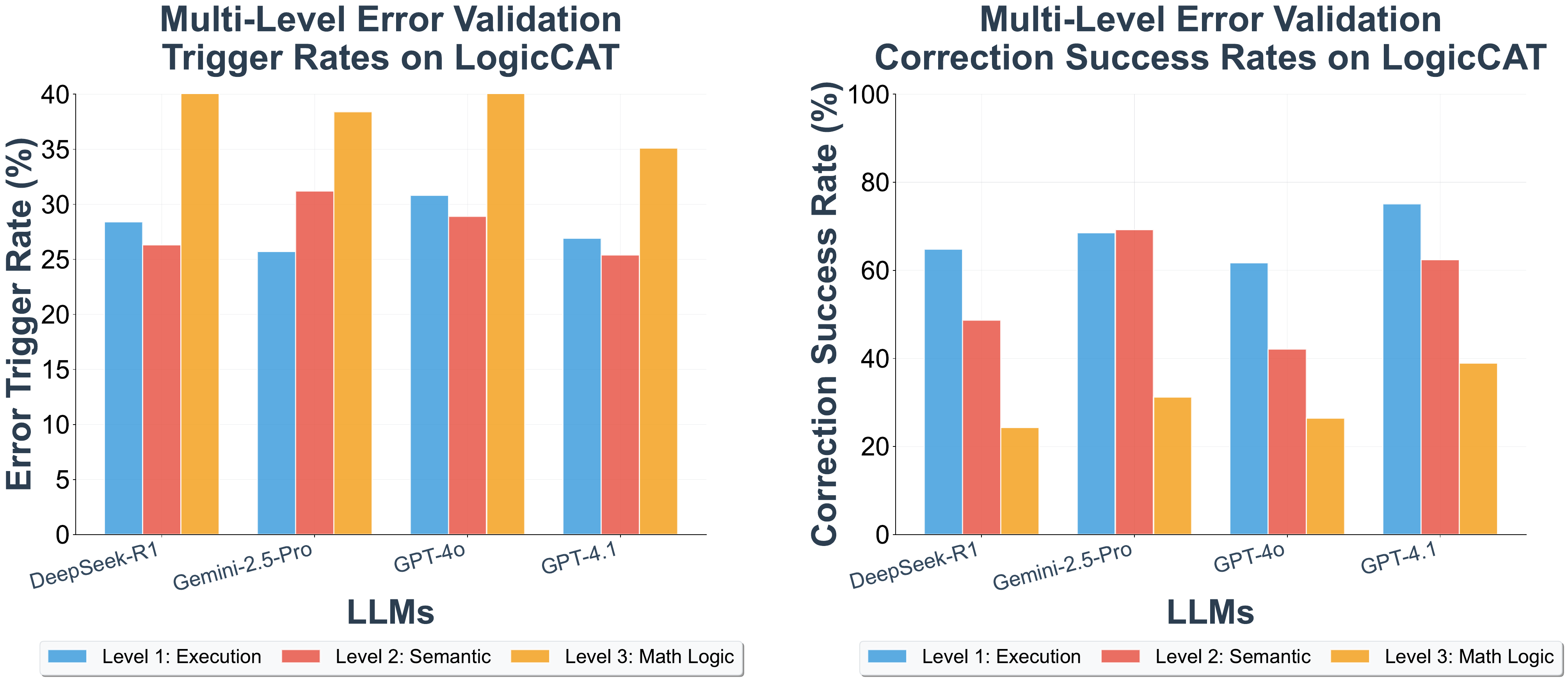}
     \caption{Analysis of the multi-level validation and correction mechanism on LogicCat. (Left) Trigger rates for each validation level, showing that semantic and mathematical logic errors are common. (Right) Correction success rates, demonstrating the effectiveness of our targeted correction strategies for each error type.}
     \label{fig:error_validation_analysis}
 \end{figure}

 \begin{table*}[htbp]
\centering

\small
\begin{tabular}{l cc cc cc cc}
\toprule
& \multicolumn{2}{c}{\textbf{DeepSeek-R1}} & \multicolumn{2}{c}{\textbf{GPT-4o}} & \multicolumn{2}{c}{\textbf{GPT-4.1}} & \multicolumn{2}{c}{\textbf{Gemini-2.5-Pro}} \\
\cmidrule(lr){2-3} \cmidrule(lr){4-5} \cmidrule(lr){6-7} \cmidrule(lr){8-9}
\textbf{Configuration} & \textbf{EX} & \textbf{Drop} & \textbf{EX} & \textbf{Drop} & \textbf{EX} & \textbf{Drop} & \textbf{EX} & \textbf{Drop} \\
\midrule
\textbf{Full Model (SteinerSQL)} & \textbf{33.98} & - & \textbf{32.00} & - & \textbf{35.04} & - & \textbf{36.10} & - \\
\midrule
\textit{Core Stage Isolation} \\
- Stage 2+3 Only (w/o Decomposition) & 27.48 & 6.50 & 27.48 & 4.52 & 30.41 & 5.37 & 29.42 & 5.68 \\
- Stage 1+3 Only (w/o Navigation) & 28.13 & 5.85 & 25.46 & 6.54 & 29.17 & 5.87 & 30.01 & 6.09 \\
- Stage 1+2 Only (w/o Validation) & 28.24 & 5.74 & 26.83 & 5.17 & 30.29 & 4.75 & 31.44 & 4.66 \\
\midrule
\textit{Stage 2 Ablation} \\
- w/o Connection Cost & 31.53 & 2.45 & 29.68 & 2.32 & 33.19 & 1.85 & 33.83 & 2.27 \\
- w/o Semantic Cost & 32.08 & 1.90 & 30.22 & 1.78 & 32.63 & 2.41 & 34.64 & 1.46 \\
- w/o Statistical Cost & 32.40 & 1.58 & 30.86 & 1.14 & 35. & 0.89 & 34.28 & 1.82 \\
- Remove Graph Module & 29.17 & 4.81 & 26.29 & 5.71 & 29.83 & 5.21 & 30.87 & 5.23 \\
- w/o Structured Examples & 29.69 & 4.29 & 26.94 & 5.06 & 29.50 & 5.54 & 31.19 & 4.91 \\
\bottomrule
\end{tabular}
\caption{Comprehensive ablation study of SteinerSQL components across four backbone models on the LogicCat dataset. We report Execution Accuracy (EX) and the absolute performance drop in percentages (\%).}
\label{tab:comprehensive_ablation_four_backbones}
\end{table*}
\subsubsection{Error Validation Analysis}
    
Our error validation analysis confirms that the primary challenges in complex Text-to-SQL are logical, not syntactic. As shown in Figure~\ref{fig:error_validation_analysis}, the high rate of semantic (Level 2) and math logic (Level 3) errors validates our central premise that failures stem from flawed mathematical decomposition and schema navigation, rendering simple execution-only validation insufficient. Our results demonstrate the practical value of multi-level validation, particularly highlighting our principled path re-planning loop as essential for resolving these deep reasoning failures where simpler feedback mechanisms would fall short. This analysis confirms that a sophisticated, multi-level approach to error validation is not merely beneficial but indispensable for tackling the true complexity of advanced Text-to-SQL tasks.

\subsection{Ablation and Component Analysis}
\label{sec:ablation}

\paragraph{Component Contribution.}
\label{sec:component}

Our ablation study, summarized in Table \ref{tab:comprehensive_ablation_four_backbones},allows us to analyze the distinct role and contribution of each component within our framework. The results confirm that while the integrated system achieves optimal performance, each stage offers significant, stand-alone value. The Schema Navigation stage is the primary driver of performance, with its removal causing the largest drop (6.09\%). This highlights its central role in structuring the reasoning path. The Math Decomposition stage provides the essential logical foundation for this navigation, and its removal also leads to a major performance decrease (5.52\%). This suggests that our decomposition approach could potentially benefit other graph-based Text-to-SQL systems by improving their input formulation. Finally, the Multi-level Validation stage acts as a powerful refinement module, responsible for a 5.08\% gain, demonstrating its utility as a post-processing mechanism. Further analysis shows that the Graph Module (5.24\% drop) and the use of Structured Examples (4.95\% drop) both serve as powerful guidance mechanisms. Within the graph, the different cost functions provide finer-grained signals for path selection. 

% In summary, rather than being a monolithic block, our framework is a well-orchestrated pipeline where each component makes a strong, hierarchical contribution. 

 \begin{figure}[t]
     \centering
     \includegraphics[width=\linewidth]{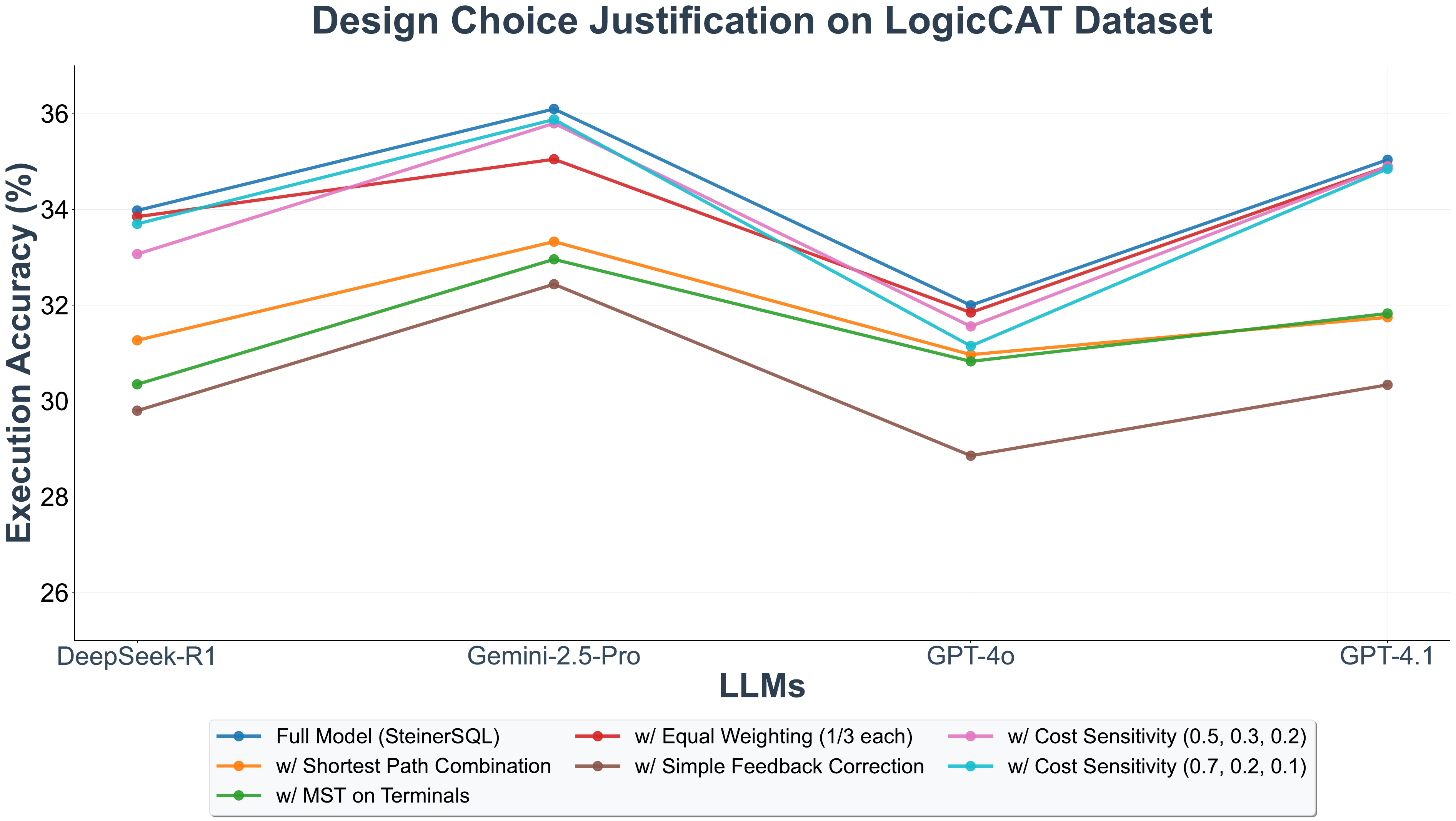}
     \caption{Design choice justification on the LogicCat dataset. The full SteinerSQL model consistently outperforms variants using alternative path-finding algorithms (Shortest Path Combination, MST on Terminals), cost sensitivity, and correction strategies (Simple Feedback). }
     \label{fig:choicejustify}
 \end{figure}

\paragraph{Design Choice Justification.}

Figure~\ref{fig:choicejustify} validates our core design choices by comparing the full model against variants using alternative algorithms.
Our model consistently outperforms simpler graph traversal algorithms, including Shortest Path Combination and MST on Terminals, by a margin of 1-3 percentage points across all tested LLMs. This validates that Steiner trees provide a more accurate model in delivering sophisticated SQL queries. Also, the use of a principled, weighted cost function describe in Section \ref{sec:cost} surpasses other weighting. Finally, our multi-level validation mechanism consistently outperforms a Simple Feedback loop, proving that a structured approach of validation is more effective.

\section{Conclusion}
In this work, we introduced SteinerSQL, a framework that bridges the critical gap between mathematical reasoning and schema navigation in complex Text-to-SQL generation. This method establishes new state-of-the-art on challenging benchmarks. Collectively, our work introduces a new, principled paradigm for complex Text-to-SQL that demonstrates the transformative potential of integrated optimization approaches in building more powerful reasoning systems for SQL generation.

\section{Limitation}
While SteinerSQL demonstrates strong performance, we identify two areas for future exploration. First, the framework's initial mathematical decomposition relies on the reasoning capabilities of the backbone LLM. Although our multi-level validation mitigates errors, advancements in more structured or hybrid decomposition techniques could further improve the initial identification of terminal tables. Second, our edge cost function, while principled, uses fixed weights. Future work could explore adaptive weighting schemes that dynamically adjust to the specific query or database schema, potentially offering more precise path-finding.

\bibliography{acl_latex}

\appendix
\section{Methodology Appendix}
\label{sec:appendixmethod}
We demonstrate the LLM agent pipeline for the SteinerSQL framework in Algorithm \ref{algo:steiner_sql} and shown the proof of the theorem proposed in Section \ref{sec:modeling}. 

\begin{algorithm*}[t]
\caption{SteinerSQL LLM Agent Framework Pipeline}
\label{algo:steiner_sql}
\begin{algorithmic}
\STATE \textbf{Input:} Natural language question $q$, Database schema $\mathcal{S} = (\mathcal{T}, \mathcal{R})$
\STATE \textbf{Output:} Valid SQL query $Q_{sql}$ or error message

\STATE \textbf{// Stage 1: Mathematical Decomposition}
\STATE $\mathcal{E}_{\text{math}} \leftarrow$ \textsc{LLM-ExtractMathEntities}($q$)
\STATE $\mathcal{T}_{\text{req}} \leftarrow$ \textsc{MathematicalDependencyAnalysis}($q, \mathcal{E}_{\text{math}}$)

\FOR{$iteration = 1$ to $3$}
    \STATE \textbf{// Stage 2: Schema Navigation}
    \STATE $G_{\mathcal{S}} \leftarrow$ \textsc{ConstructSchemaGraph}($\mathcal{S}$) \COMMENT{Create vertices $V = \mathcal{T}$}
    \STATE Add edges $(t_i, t_j)$ if FK exists or $\text{sim}(t_i, t_j) \geq 0.75$
    \STATE Compute edge costs: $C = 0.4 \cdot C_{\text{connect}} + 0.4 \cdot C_{\text{sem}} + 0.2 \cdot C_{\text{stat}}$
    
    \STATE \textbf{// Steiner Tree Solution}
    \STATE $G'_{\mathcal{S}} \leftarrow$ \textsc{FloydWarshall}($G_{\mathcal{S}}$) \COMMENT{Metric closure}
    \STATE $MST \leftarrow$ \textsc{MinSpanningTree}($G'_{\mathcal{S}}[\mathcal{T}_{\text{req}}]$)
    \STATE $T_{steiner} \leftarrow$ \textsc{MapToOriginalPaths}($MST, G_{\mathcal{S}}$) and prune cycles
    
    \STATE \textbf{// LLM Query Generation with Structured Prompt}
    \STATE $prompt \leftarrow$ \textsc{BuildSystemPrompt}($T_{steiner}$) \COMMENT{5 components: role, requirements, relations, plan, guidelines}
    \STATE $Q_{sql} \leftarrow$ \textsc{LLM-GenerateSQL}($prompt, q$)
    
    \STATE \textbf{// Stage 3: Multi-level Validation}
    \STATE $level1 \leftarrow$ \textsc{CheckSyntaxAndExecutability}($Q_{sql}, \mathcal{S}$)
    \IF{NOT $level1$}
        \STATE \textbf{return} "Syntax Error"
    \ENDIF
    
    \STATE $level2 \leftarrow$ \textsc{CheckSemanticConsistency}($Q_{sql}, q, \mathcal{T}_{\text{req}}$)
    \STATE $level3 \leftarrow$ \textsc{CheckMathLogic}($Q_{sql}, q, \mathcal{E}_{\text{math}}$)
    
    \IF{$level2$ AND $level3$}
        \STATE \textbf{return} $Q_{sql}$
    \ELSE
        \STATE $\mathcal{T}_{\text{req}} \leftarrow$ \textsc{UpdateTerminals}($\mathcal{T}_{\text{req}}, level2, level3$)
    \ENDIF
\ENDFOR

\STATE \textbf{return} "Max iterations reached"
\end{algorithmic}
\end{algorithm*}

\subsection{Proof}
\label{app:steiner-proof}

We formalize the correspondence between minimum Steiner trees on the schema graph and optimal reasoning scaffolds for answering a mathematical query. Let $G_{\mathcal S}=(V,E,C)$ be the schema graph from the main text. For notational convenience, we set $w \equiv C$ so that $w(e)=C(e)$ for all $e\in E$, and we assume $w(e)\ge 0$. We also write
\[
C_{\text{total}}(H)\;:=\;\sum_{e\in E(H)} w(e)
\]
for any subgraph $H\subseteq G_{\mathcal S}$.

Edges $E$ comprise (i) declared foreign-key/primary-key (or other well-typed equi-join) links and (ii) reproducible similarity-based links whose maximum column-name/type similarity is at least $\tau$, as in Definition~1 of the main text.

All statements below are conditioned on Assumption~1 (A1--A4) from the main text: schema fidelity, terminal soundness, additive and monotone surrogate edge costs, and operation locality.

\noindent\emph{Existence.}
Since $|V|<\infty$ and $w(e)\ge 0$, a minimum-cost connected subgraph spanning $\mathcal T_{\text{req}}$ (hence a minimum Steiner tree) exists.

\paragraph{Assumption.}
We rely on the following standard conditions.
\begin{enumerate}
\item \textbf{Schema fidelity.} If a set of tables induces a connected subgraph of $G_{\mathcal S}$, then there exists a well-typed equi-join plan over those tables; conversely, every well-typed join plan over base tables uses only edges present in $G_{\mathcal S}$.
\item \textbf{Terminal soundness.} Every correct answer to the query depends on attributes drawn from all tables in $\mathcal{T}_{\text{req}}$ (no proper subset suffices).
\item \textbf{Additive, monotone surrogate cost.} The complexity/efficiency proxy of a reasoning scaffold decomposes additively over used joins:
\[
\mathrm{cost}(\text{scaffold}) \;=\; \sum_{e\in E(\text{scaffold})} w(e),
\]
and adding edges/joins does not decrease the cost (i.e., the surrogate is monotone in edge inclusion).
\item \textbf{Operation locality.} Selections, projections, grouping, and aggregation required by the query can be implemented on top of any join plan that correctly connects $\mathcal{T}_{\text{req}}$; they do not require additional base tables beyond those appearing in the scaffold.
\end{enumerate}

\begin{definition}[Reasoning scaffold]
A reasoning scaffold for the query is any connected subgraph $H=(V_H,E_H)$ of $G_{\mathcal S}$ such that $\mathcal{T}_{\text{req}}\subseteq V_H$ and there exists a well-typed SQL plan whose FROM/JOIN component uses exactly $E_H$ and whose SELECT/WHERE/GROUP BY/HAVING components implement the query's non-join operators. An \emph{optimal} scaffold is one minimizing $\mathrm{cost}(H)=\sum_{e\in E_H} w(e)$.
\end{definition}

\begin{definition}[Steiner tree]
A (vertex-)Steiner tree for $(G_{\mathcal S},\mathcal{T}_{\text{req}})$ is a connected, cycle-free subgraph $T=(V_T,E_T)$ of $G_{\mathcal S}$ that spans $\mathcal{T}_{\text{req}}$ and is of minimum total weight $\sum_{e\in E_T} w(e)$ among all connected subgraphs that contain $\mathcal{T}_{\text{req}}$.
\end{definition}

We prove the theorem via (i) soundness of Steiner trees as valid scaffolds, (ii) completeness of scaffolds as terminal-connecting subgraphs, and (iii) optimality via cost alignment.

\subsubsection{Soundness: Steiner trees yield valid scaffolds}
\begin{lemma}[Soundness]
\label{lem:soundness}
Let $T$ be any Steiner tree for $(G_{\mathcal S},\mathcal{T}_{\text{req}})$. Then $T$ is a valid reasoning scaffold.
\end{lemma}

\begin{proof}
By definition, $T$ is connected and $\mathcal{T}_{\text{req}}\subseteq V_T$. By (A1) Schema fidelity, any connected subgraph of $G_{\mathcal S}$ induces a well-typed equi-join plan over its vertices using exactly $E_T$. By (A4) Operation locality, the non-join operators of the mathematical query (selections, projections, grouping, aggregation) can be layered on top of this join plan without introducing new base tables. Hence there exists a complete SQL plan whose JOIN component uses $E_T$ and that computes the desired quantity, making $T$ a valid scaffold.
\end{proof}

\subsubsection{Completeness: Every valid scaffold connects the terminals}
\begin{lemma}[Completeness]
\label{lem:completeness}
Let $H$ be any valid reasoning scaffold. Then $H$ is a connected subgraph of $G_{\mathcal S}$ that spans $\mathcal{T}_{\text{req}}$.
\end{lemma}

\begin{proof}
By definition of scaffold and (A1), $H$ is a subgraph of $G_{\mathcal S}$ over which a well-typed join plan exists; such a plan requires connectivity among the used tables, hence $H$ is connected. By (A2) Terminal soundness, any correct answer depends on all tables in $\mathcal{T}_{\text{req}}$, so these must appear in the FROM/JOIN portion; therefore $\mathcal{T}_{\text{req}}\subseteq V_H$.
\end{proof}

\subsubsection{Cost preservation and cycle pruning}
\begin{lemma}[Cycle pruning]
\label{lem:cycle-pruning}
\begin{proof}
If $H$ contains a cycle $C$, remove any edge $e\in C$. In an undirected graph, deleting one edge on a cycle preserves connectivity; thus $H\setminus\{e\}$ remains connected and still spans $\mathcal T_{\text{req}}$.
Because $w(e)\ge 0$, the cost does not increase and strictly decreases if $w(e)>0$.
Iterate until no cycles remain to obtain $H'$.
\end{proof}

\end{lemma}

\begin{proof}
If $H$ contains a cycle $C$, remove any edge $e\in C$ that does not disconnect $\mathcal{T}_{\text{req}}$ (such an edge exists because cycles have redundant connectivity). Since $w(e)\ge 0$ by (A3), removing $e$ does not increase cost and strictly decreases it if $w(e)>0$. Repeat until no cycles remain. Connectivity and coverage of $\mathcal{T}_{\text{req}}$ are preserved at each step, yielding $H'$ with $\mathrm{cost}(H')\le \mathrm{cost}(H)$.
\end{proof}

\begin{lemma}[Cost preservation from plans to graphs]
\label{lem:cost-alignment}
Under (A3), the surrogate cost of any valid reasoning scaffold equals the sum of weights over its join edges, and adding superfluous joins (edges) cannot reduce this cost.
\end{lemma}

\begin{proof}
Immediate from additivity and monotonicity in (A3).
\end{proof}

\subsubsection{Optimality and equivalence}
\begin{proposition}[From Steiner optimality to scaffold optimality]
\label{prop:st-to-sc}
Let $T$ be a minimum Steiner tree for $(G_{\mathcal S},\mathcal{T}_{\text{req}})$. Then $T$ achieves the minimum surrogate cost among all valid reasoning scaffolds.
\end{proposition}

\begin{proof}
By Lemma~\ref{lem:soundness}, $T$ is a valid scaffold. Consider any valid scaffold $H$. By Lemma~\ref{lem:completeness}, $H$ is a connected subgraph spanning $\mathcal{T}_{\text{req}}$. By Lemma~\ref{lem:cycle-pruning}, there exists a cycle-free $H'\subseteq H$ that still spans $\mathcal{T}_{\text{req}}$ with $\mathrm{cost}(H')\le \mathrm{cost}(H)$. Since $T$ is the minimum-cost connected, cycle-free subgraph spanning $\mathcal{T}_{\text{req}}$, we have $\mathrm{cost}(T)\le \mathrm{cost}(H')\le \mathrm{cost}(H)$. By Lemma~\ref{lem:cost-alignment}, these costs coincide with surrogate scaffold costs. Hence no valid scaffold beats $T$.
\end{proof}

\begin{proposition}[From scaffold optimality to Steiner optimality]
\label{prop:sc-to-st}
If $H^\star$ is an optimal reasoning scaffold, then there exists a minimum Steiner tree $T$ with $\mathrm{cost}(T)=\mathrm{cost}(H^\star)$.
\end{proposition}

\begin{proof}
By Lemma~\ref{lem:completeness}, $H^\star$ is a connected subgraph spanning $\mathcal{T}_{\text{req}}$. Apply Lemma~\ref{lem:cycle-pruning} to obtain a cycle-free $H' \subseteq H^\star$ with $\mathrm{cost}(H')\le \mathrm{cost}(H^\star)$. Optimality of $H^\star$ forces equality, so $H'$ is also optimal among connected subgraphs spanning $\mathcal{T}_{\text{req}}$. Hence $H'$ is a minimum Steiner tree $T$.
\end{proof}

\begin{proof}[Proof of Theorem]
By Proposition~\ref{prop:st-to-sc}, any minimum Steiner tree $T$ is an optimal reasoning scaffold under the surrogate cost. Conversely, by Proposition~\ref{prop:sc-to-st}, any optimal scaffold collapses (by cycle pruning) to a minimum Steiner tree of equal cost. Together with (A4) ensuring that non-join operators do not require additional base tables beyond those in the scaffold, this proves the equivalence and establishes that the minimum Steiner tree corresponds exactly to the most efficient and relevant scaffold.
\end{proof}

\paragraph{Remarks and edge cases.}
(1) \emph{Non-uniqueness.} If multiple minimum Steiner trees exist, each is an optimal scaffold under the surrogate; downstream tie-breaking (e.g., favoring lower estimated cardinality) can be applied without affecting Theorem. (2) \emph{Bridge tables as Steiner nodes.} Many-to-many link tables naturally appear as Steiner vertices; their inclusion is compelled by connectivity and captured by the cost. (3) \emph{Self-joins and aliases.} If a base table must be joined to itself, represent each alias as a distinct vertex with an alias edge to preserve correctness of the mapping; the above arguments remain unchanged. (4) \emph{Selections and filters.} Predicate pushdown may \emph{reduce} runtime but does not change the necessity of connecting $\mathcal{T}_{\text{req}}$; our result concerns the structural scaffold. 

\paragraph{Scope and limitations.}
Theorem concerns the JOIN skeleton only: selections/projections/grouping/aggregation are implemented on top of the scaffold without introducing extra base tables (operation locality).
Cases requiring external sources or materialized auxiliaries (e.g., complex windowed analytics) are outside our guarantee.
Moreover, if Stage~1 under- or over-identifies $\mathcal T_{\text{req}}$, the selected scaffold may deviate; our validation and re-planing loop updates $\mathcal T_{\text{req}}$ when inconsistencies are detected.

\section{Experiment Details}
\label{sec:exp-details}

\paragraph{Hardware \& Operation.}
All experiments ran on a single node with 2$\times$ NVIDIA A100 GPUs, 64 CPU cores, and 512 GB RAM with Ubuntu 22.04.3 LTS, CUDA12.1.

\paragraph{Software Stack.}
Python~3.10.13; PyTorch~2.3.1+cu121; transformers~4.43; accelerate~0.31; datasets~2.19; sentence-transformers~2.7; networkx~3.3; sqlglot~23.x; sqlite3~3.45. 

\paragraph{Decoding and Runtime Configuration.}
Unless specified otherwise, we use \emph{greedy decoding}:
\begin{itemize}
    \item temperature 0.0, top\_p/top\_k disabled; repetition/presence/frequency penalties 1.0/0/0;
    \item max input context 32,768 tokens (subject to model limits), max new tokens 512;
    \item stop sequences: \texttt{\{";\textbackslash n--END--", "\textless{}eot\_id\textgreater{}"\}}.
\end{itemize}
Batched inference is used only where supported by the provider; otherwise, calls are sequential. Concurrency is capped to respect provider rate limits; failed calls are retried up to three times with exponential backoff.

\section{Ablation Results}
\label{app:full_ablation_results}

To demonstrate that the contributions of SteinerSQL's components are robust and not specific to a single dataset, we replicated our main ablation study across all three other benchmarks:BIRD-dev, Spider1.0 and Spider2.0-Lite benchmarks. The results, presented in Tables \ref{tab:ablation_bird}, \ref{tab:ablation_spider2} and \ref{tab:ablation_spider}, show generally consistent trends of ablation.

\begin{table*}[h!]
\centering
\caption{Ablation study on BIRD-dev dataset across four LLM backbones. We report Execution Accuracy (EX) in percentages (\%) and the absolute performance drop compared to the full model.}
\label{tab:ablation_bird}
\small
\begin{tabular}{l cc cc cc cc}
\toprule
& \multicolumn{2}{c}{\textbf{DeepSeek-R1}} & \multicolumn{2}{c}{\textbf{GPT-4o}} & \multicolumn{2}{c}{\textbf{GPT-4.1}} & \multicolumn{2}{c}{\textbf{Gemini-2.5-Pro}} \\
\cmidrule(lr){2-3} \cmidrule(lr){4-5} \cmidrule(lr){6-7} \cmidrule(lr){8-9}
\textbf{Configuration} & \textbf{EX} & \textbf{Drop} & \textbf{EX} & \textbf{Drop} & \textbf{EX} & \textbf{Drop} & \textbf{EX} & \textbf{Drop} \\
\midrule
\textbf{Full Model (SteinerSQL)} & \textbf{70.01} & - & \textbf{64.93} & - & \textbf{67.08} & - & \textbf{73.92} & - \\
\midrule
\textit{Core Stage Isolation} \\
- Stage 2+3 Only (w/o Decomposition) & 61.80 & 8.21 & 57.30 & 7.63 & 60.17 & 6.91 & 65.19 & 8.73 \\
- Stage 1+3 Only (w/o Navigation) & 65.45 & 4.56 & 61.08 & 3.85 & 63.89 & 3.19 & 70.14 & 3.78 \\
- Stage 1+2 Only (w/o Validation) & 63.10 & 6.91 & 58.80 & 6.13 & 59.71 & 7.37 & 67.01 & 6.91 \\
\midrule
\textit{Stage 2 Ablation} \\
- w/o Connection Cost & 68.78 & 1.23 & 65.15 & -0.98 & 65.71 & 1.37 & 72.82 & 1.10 \\
- w/o Semantic Cost & 69.10 & 0.91 & 63.23 & 1.70 & 68.12 & -1.04 & 72.56 & 1.36 \\
- w/o Statistical Cost & 70.14 & -0.13 & 64.99 & -0.06 & 65.91 & 1.17 & 73.99 & -0.07 \\
- Remove Graph Module & 64.86 & 5.15 & 58.93 & 6.00 & 61.02 & 6.06 & 68.19 & 5.73 \\
- w/o Structured Examples & 65.32 & 4.69 & 59.26 & 5.67 & 62.13 & 4.95 & 69.17 & 4.75 \\
\bottomrule
\end{tabular}
\end{table*}

\begin{table*}[h!]
\centering
\caption{Ablation study on Spider2.0-Lite dataset across four LLM backbones. We report Execution Accuracy (EX) in percentages (\%) and the absolute performance drop compared to the full model.}
\label{tab:ablation_spider2}
\small
\begin{tabular}{l cc cc cc cc}
\toprule
& \multicolumn{2}{c}{\textbf{DeepSeek-R1}} & \multicolumn{2}{c}{\textbf{GPT-4o}} & \multicolumn{2}{c}{\textbf{GPT-4.1}} & \multicolumn{2}{c}{\textbf{Gemini-2.5-Pro}} \\
\cmidrule(lr){2-3} \cmidrule(lr){4-5} \cmidrule(lr){6-7} \cmidrule(lr){8-9}
\textbf{Configuration} & \textbf{EX} & \textbf{Drop} & \textbf{EX} & \textbf{Drop} & \textbf{EX} & \textbf{Drop} & \textbf{EX} & \textbf{Drop} \\
\midrule
\textbf{Full Model (SteinerSQL)} & \textbf{34.92} & - & \textbf{25.05} & - & \textbf{26.33} & - & \textbf{40.04} & - \\
\midrule
\textit{Core Stage Isolation} \\
- Stage 2+3 Only (w/o Decomposition) & 27.61 & 7.31 & 19.93 & 5.12 & 19.93 & 6.40 & 33.09 & 6.95 \\
- Stage 1+3 Only (w/o Navigation) & 29.80 & 5.12 & 21.21 & 3.84 & 22.67 & 3.66 & 35.28 & 4.76 \\
- Stage 1+2 Only (w/o Validation) & 28.15 & 6.77 & 19.01 & 6.04 & 19.56 & 6.77 & 33.46 & 6.58 \\
\midrule
\textit{Stage 2 Ablation} \\
- w/o Connection Cost & 34.92 & 0.00 & 23.95 & 1.10 & 26.51 & -0.18 & 39.12 & 0.92 \\
- w/o Semantic Cost & 33.46 & 1.46 & 23.40 & 1.65 & 27.80 & -1.47 & 39.55 & 1.49 \\
- w/o Statistical Cost & 34.82 & -1.10 & 24.13 & 0.92 & 25.78 & 0.55 & 38.94 & 1.10 \\
- Remove Graph Module & 29.62 & 5.30 & 18.83 & 6.22 & 20.66 & 5.67 & 34.19 & 5.85 \\
- w/o Structured Examples & 30.16 & 4.76 & 19.38 & 5.67 & 20.84 & 5.49 & 34.92 & 5.12 \\
\bottomrule
\end{tabular}
\end{table*}

\begin{table*}[h!]
\centering
\caption{Ablation study on Spider-test dataset across four LLM backbones. We report Execution Accuracy (EX) in percentages (\%) and the absolute performance drop compared to the full model.}
\label{tab:ablation_spider}
\small
\begin{tabular}{l cc cc cc cc}
\toprule
& \multicolumn{2}{c}{\textbf{DeepSeek-R1}} & \multicolumn{2}{c}{\textbf{GPT-4o}} & \multicolumn{2}{c}{\textbf{GPT-4.1}} & \multicolumn{2}{c}{\textbf{Gemini-2.5-Pro}} \\
\cmidrule(lr){2-3} \cmidrule(lr){4-5} \cmidrule(lr){6-7} \cmidrule(lr){8-9}
\textbf{Configuration} & \textbf{EX} & \textbf{Drop} & \textbf{EX} & \textbf{Drop} & \textbf{EX} & \textbf{Drop} & \textbf{EX} & \textbf{Drop} \\
\midrule
\textbf{Full Model (SteinerSQL)} & \textbf{88.59} & - & \textbf{87.43} & - & \textbf{87.23} & - & \textbf{88.30} & - \\
\midrule
\textit{Core Stage Isolation} \\
- Stage 2+3 Only (w/o Decomposition) & 80.56 & 8.03 & 80.46 & 6.97 & 80.27 & 6.96 & 79.30 & 9.00 \\
- Stage 1+3 Only (w/o Navigation) & 84.62 & 3.97 & 83.46 & 3.97 & 84.24 & 2.99 & 83.27 & 5.03 \\
- Stage 1+2 Only (w/o Validation) & 82.59 & 6.00 & 81.43 & 6.00 & 80.27 & 6.96 & 82.30 & 6.00 \\
\midrule
\textit{Stage 2 Ablation} \\
- w/o Connection Cost & 89.07 & -0.48 & 86.46 & 0.97 & 86.07 & 1.16 & 87.33 & 0.97 \\
- w/o Semantic Cost & 87.62 & 0.97 & 85.98 & 1.45 & 86.27 & 0.96 & 87.14 & 1.16 \\
- w/o Statistical Cost & 87.62 & 0.97 & 88.39 & -0.96 & 86.75 & 0.48 & 89.46 & -1.16 \\
- Remove Graph Module & 83.56 & 5.03 & 81.43 & 6.00 & 81.24 & 5.99 & 83.27 & 5.03 \\
- w/o Structured Examples & 84.62 & 3.97 & 82.40 & 5.03 & 82.21 & 5.02 & 84.33 & 3.97 \\
\bottomrule
\end{tabular}
\end{table*}

\section{Cases}
\label{append:case}
Table \ref{tab:caselogic1} and \ref{tab:caselogic2} are execution pipeline example in LogicCat, and the question is ``If the status of a data collector shows 'shutdown' and its installation altitude is 3000 meters, assuming the last data collection record before shutdown indicates a temperature of -10°C, please calculate the atmospheric pressure value at that location and analyze possible reasons for the shutdown." Table \ref{tab:casespider1} and \ref{tab:casespider2} are execution pipeline example in Spider2.0-Lite, and the question is ``Between April 1 and July 31 of 2017, using the hits product revenue data along with the totals transactions to classify sessions as purchase (transactions $\geq$ 1 and productRevenue not null) or non-purchase (transactions null and productRevenue null), compare the average pageviews per visitor for each group by month."

\begin{table*}[h!]
\centering
\caption{\textbf{Case of SteinerSQL Pipeline Execution in LogicCat (Part 1): Query and Decomposition}}
\label{tab:caselogic1}
\begin{tabular}{@{}lp{0.6\textwidth}@{}}
\toprule
\multicolumn{2}{@{}l}{\textbf{SteinerSQL Pipeline Execution for Collector bq004}} \\
\midrule
\textbf{Component} & \textbf{Details and Values} \\
\midrule
\textbf{Question} & If the status of a data collector shows 'shutdown' and its installation altitude is 3000 meters, assuming the last data collection record before shutdown indicates a temperature of $-10$°C, please calculate the atmospheric pressure value at that location and analyze possible reasons for the shutdown. \\
\addlinespace
\textbf{Final Procedure} & \texttt{/* Fictional Procedure for Calculation and Analysis */ \newline 1. SELECT installation\_altitude\_m, last\_temp\_c FROM collectors WHERE status = 'shutdown' AND collector\_id = 'bq004'; \newline 2. INPUT altitude (h) and temperature (T) INTO Barometric\_Formula(h, T); \newline 3. RETURN calculated\_pressure; \newline 4. ANALYZE environmental\_factors (T, h, calculated\_pressure) FOR shutdown\_causation;} \\
\addlinespace
\midrule
\multicolumn{2}{@{}l}{\textbf{Stage 1: Mathematical Decomposition}} \\
\midrule
\textbf{1a. Mathematical} &
- \textbf{Calculation:} Atmospheric pressure using the Barometric Formula: $P = P_0 \cdot e^{-\frac{gMh}{RT}}$ \newline
- \textbf{Constants:}
    - Sea level pressure ($P_0$): $101325$ Pa
    - Gravitational acceleration ($g$): $9.80665$ m/s²
    - Molar mass of Earth's air ($M$): $0.0289644$ kg/mol
    - Universal gas constant ($R$): $8.31446$ J/(mol$\cdot$K) \newline
- \textbf{Variables:}
    - Altitude ($h$): $3000$ m
    - Temperature ($T$): $-10$°C, which must be converted to Kelvin ($263.15$ K). \newline
- \textbf{Analysis:} Qualitative assessment of failure modes based on environmental data. \\
\addlinespace
\textbf{1b. Computational} &
- The calculation requires \texttt{installation\_altitude} and \texttt{last\_temperature}. \newline
- The analysis requires the \texttt{status}, \texttt{last\_temperature}, and the calculated pressure. \\
\addlinespace
\textbf{1c. Terminal Table} &
- A conceptual schema is assumed for this problem. \newline
- \texttt{collectors} (contains \texttt{collector\_id}, \texttt{status}, \texttt{installation\_altitude\_m}) \newline
- \texttt{readings} (contains \texttt{collector\_id}, \texttt{timestamp}, \texttt{temperature\_c}) \newline
\textbf{Resulting Terminals $\mathcal{T}_{\text{req}}$:} $\{$\texttt{collectors}, \texttt{readings}$\}$ \\
\bottomrule
\end{tabular}
\end{table*}

\begin{table*}[h!]
\centering
\caption{\textbf{Case of SteinerSQL Pipeline Execution in LogicCat (Part 2): Navigation and Validation}}
\label{tab:caselogic2}
\begin{tabular}{@{}lp{0.7\textwidth}@{}}
\toprule
\multicolumn{2}{@{}l}{\textbf{Stage 2: Schema Navigation}} \\
\midrule
\textbf{2a. Schema Graph} &
- \textbf{Vertices $V$:} $\{$\texttt{collectors}, \texttt{readings}$\}$ \newline
- \textbf{Edges $E$ and Costs $C(e)$:}
    - A direct foreign key relationship exists between the tables.
    - $C(\text{collectors, readings}) = 0.05$ (Direct FK link, cost is very low). \\
\addlinespace
\textbf{2b. Steiner Tree} &
- \textbf{Problem:} Find the path to link collector properties with their latest readings. \newline
- \textbf{Terminals:} $\{$\texttt{collectors}, \texttt{readings}$\}$ \newline
- \textbf{Reasoning Scaffold:} The resulting tree is a direct join: \textbf{\texttt{collectors} -- \texttt{readings}}. This links the static altitude data to the last recorded temperature. \\
\addlinespace
\midrule
\multicolumn{2}{@{}l}{\textbf{Stage 3: Multi-level Validation \& Execution}} \\
\midrule
\textbf{3a. Level 1:} &
- \textbf{Check:} Is the procedural logic sound? \newline
- \textbf{Result:} Pass. The procedure correctly identifies the necessary inputs, applies the correct physical formula, and outlines an analysis stage. \\
\addlinespace
\textbf{3b. Level 2:} &
- \textbf{Check:} Does the procedure align with the user's intent? \newline
- \textbf{Result:} Pass. It directly addresses both parts of the question: calculating the pressure and analyzing the shutdown. \\
\addlinespace
\textbf{3c. Level 3:} &
- \textbf{Check:} Is the mathematical logic and final analysis correct? \newline
- \textbf{Result:} Pass.
    - \textbf{Calculation:}
        - $T = -10 + 273.15 = 263.15$ K
        - Exponent: $-\frac{(9.80665 \cdot 0.0289644 \cdot 3000)}{(8.31446 \cdot 263.15)} = -\frac{851.65}{2187.95} \approx -0.3892$
        - $P = 101325 \cdot e^{-0.3892} = 101325 \cdot 0.6776 \approx 68669.5$ Pa
        - Converting to hectopascals: $68669.5 \text{ Pa} = \textbf{686.7 hPa}$. (Note: Using a standard atmosphere model gives $\sim$700 hPa; this direct calculation is more precise for the given temperature). \newline
    - \textbf{Shutdown Analysis:}
        - \textbf{Low Temperature Failure:} At $-10$°C, battery performance drops significantly and can lead to a low-voltage shutdown. This is the most probable cause.
        - \textbf{Icing:} Moisture could freeze on sensors, anemometers, or solar panels (if present), causing mechanical failure or power loss.
        - \textbf{Component Stress:} Extreme cold can make materials like plastics and wiring insulation brittle, leading to physical damage. \\
\bottomrule
\end{tabular}
\end{table*}

\begin{table*}[h!]
\centering
\caption{\textbf{Case of SteinerSQL Pipeline Execution in Spider2.0-Lite (Part 1): Query and Decomposition}}
\label{tab:casespider1}
\begin{tabular}{@{}lp{0.6\textwidth}@{}}
\toprule
\textbf{Component} & \textbf{Details and Values} \\
\midrule
\textbf{Question} & Between April 1 and July 31 of 2017, using the hits product revenue data along with the totals transactions to classify sessions as purchase (transactions $\geq$ 1 and productRevenue not null) or non-purchase (transactions null and productRevenue null), compare the average pageviews per visitor for each group by month. \\
\addlinespace
\textbf{Final Query} & \texttt{SELECT FORMAT\_DATE('\%Y\%m', PARSE\_DATE('\%Y\%m\%d', date)) AS month, CASE WHEN totals.transactions >= 1 AND hits.productRevenue IS NOT NULL THEN 'Purchase' WHEN totals.transactions IS NULL AND hits.productRevenue IS NULL THEN 'Non-Purchase' END AS session\_type, SUM(totals.pageviews) / COUNT(DISTINCT fullVisitorId) AS avg\_pageviews\_per\_visitor \newline FROM \`bigquery-public-data.google\_analytics
\newline \_sample.ga\_sessions\_*\`, UNNEST(hits) AS hits \newline WHERE \_TABLE\_SUFFIX BETWEEN '20170401' AND '20170731' AND ((totals.transactions >= 1 AND hits.productRevenue IS NOT NULL) OR (totals.transactions IS NULL AND hits.productRevenue IS NULL)) \newline GROUP BY month, session\_type \newline ORDER BY month, session\_type;} \\
\addlinespace
\midrule
\multicolumn{2}{@{}l}{\textbf{Stage 1: Mathematical Decomposition}} \\
\midrule
\textbf{1a. Mathematical} &
- \textbf{Aggregation:} Average pageviews per visitor ($\text{SUM(pageviews)} / \text{COUNT(DISTINCT fullVisitorId)}$). \newline
- \textbf{Grouping:} By month and by a new classification group. \newline
- \textbf{Temporal Filter:} Date between April 1, 2017 and July 31, 2017. \newline
- \textbf{Classification Logic (CASE WHEN):}
 - Group 1 ('Purchase'): \texttt{transactions >= 1} AND \texttt{productRevenue IS NOT NULL}.
- Group 2 ('Non-Purchase'): \texttt{transactions IS NULL} AND \texttt{productRevenue IS NULL}. \\
\addlinespace
\textbf{1b. Computational} &
- The final average metric requires \texttt{pageviews}, and \texttt{fullVisitorId}. 
- The classification requires \texttt{transactions} and \texttt{productRevenue}. \newline
- The grouping requires the \texttt{date} (to extract the month) and the classification result. \newline
- The filtering requires the \texttt{date}. \\
\addlinespace
\textbf{1c. Terminal Table} &
Based on the required attributes, the essential conceptual tables (terminals) are identified. In the context of the GA schema, these are fields within the main \texttt{ga\_sessions} table. \newline
- \texttt{ga\_sessions} (contains \texttt{date}, \texttt{fullVisitorId}) \newline
- \texttt{totals} (contains \texttt{transactions}, \texttt{pageviews}) \newline
- \texttt{hits} (contains \texttt{productRevenue}) \newline
\textbf{Resulting Terminals $\mathcal{T}_{\text{req}}$:} $\{$\texttt{ga\_sessions}, \texttt{totals}, \texttt{hits}$\}$ \\
\bottomrule
\end{tabular}
\end{table*}

\begin{table*}[h!]
\centering
\caption{\textbf{Case of SteinerSQL Pipeline Execution in Spider2.0-Lite (Part 2): Navigation and Validation}}
\label{tab:casespider2}
\begin{tabular}{@{}lp{0.6\textwidth}@{}}
\toprule
\textbf{Component} & \textbf{Details and Values} \\
\midrule
\multicolumn{2}{@{}l}{\textbf{Stage 2: Schema Navigation}} \\
\midrule
\textbf{2a. Schema Graph} &
- \textbf{Vertices $V$:} $\{$\texttt{ga\_sessions}, \texttt{totals}, \texttt{hits}$\}$ \newline
- \textbf{Edges $E$ and Costs $C(e)$:}
 - Edges exist for foreign key (FK) relationships or high similarity. In the BigQuery schema, \texttt{totals} and \texttt{hits} are nested structures within \texttt{ga\_sessions}, implying a direct, low-cost link.
 - $C(\text{ga\_sessions, totals}) = 0.08$ (Direct FK-like link, cost is very low).
 - $C(\text{ga\_sessions, hits}) = 0.09$ (Direct FK-like link, cost is very low).
 - $C(\text{totals, hits}) = 0.58$ (No direct link, higher cost based on dissimilarity). \\
\addlinespace
\textbf{2b. Steiner Tree} &
- \textbf{Problem:} Find the minimum cost subgraph connecting all terminals in $\mathcal{T}_{\text{req}}$. \newline
- \textbf{Terminals:} $\{$\texttt{ga\_sessions}, \texttt{totals}, \texttt{hits}$\}$ (all nodes in the graph). \newline
- \textbf{Solution:} The problem becomes finding the Minimum Spanning Tree (MST) of the graph. The MST algorithm selects the edges with the lowest costs that connect all vertices without forming a cycle.
    - Select edge (\texttt{ga\_sessions}, \texttt{totals}) with cost $0.08$.
    - Select edge (\texttt{ga\_sessions}, \texttt{hits}) with cost $0.09$. \newline
- \textbf{Reasoning Scaffold:} The resulting tree is \textbf{\texttt{totals} -- \texttt{ga\_sessions} -- \texttt{hits}}. This scaffold indicates that \texttt{ga\_sessions} is the central entity required to link information from \texttt{totals} and \texttt{hits}. \\
\addlinespace
\midrule
\multicolumn{2}{@{}l}{\textbf{Stage 3: Multi-level Validation}} \\
\midrule
\textbf{3a. Level 1:} &
- \textbf{Check:} Is the SQL syntactically correct and executable? \newline
- \textbf{Result:} Pass. The query uses valid BigQuery SQL syntax, including \texttt{FORMAT\_DATE}, \texttt{PARSE\_DATE}, \texttt{UNNEST}, and filtering on the \texttt{\_TABLE\_SUFFIX} pseudo-column. \\
\addlinespace
\textbf{3b. Level 2:} &
- \textbf{Check:} Does the query align with the user's semantic intent? \newline
- \textbf{Result:} Pass.
 - All terminal concepts (\texttt{ga\_sessions}, \texttt{totals}, \texttt{hits}) are correctly utilized in the query.
 - The join logic (via \texttt{UNNEST}) correctly reflects the \texttt{ga\_sessions} -- \texttt{hits} path from the scaffold.
 - Attributes in \texttt{SELECT}, \texttt{WHERE}, and \texttt{GROUP BY} correctly map to the question's entities (month, classification, pageviews, etc.). \\
\addlinespace
\textbf{3c. Level 3:} &
- \textbf{Check:} Is the query's mathematical logic correct? \newline
- \textbf{Result:} Pass.
 - The \texttt{AVG} is correctly implemented as \texttt{SUM(...) / COUNT(DISTINCT ...)}.
- The \texttt{CASE WHEN} conditions for 'Purchase' and 'Non-Purchase' sessions exactly match the logic specified in the prompt.
- The date range filter in the \texttt{WHERE} clause correctly implements "Between April 1 and July 31 of 2017".
- The \texttt{GROUP BY} clause correctly includes both the month and the derived session type, as required to compare the groups. \\
\bottomrule
\end{tabular}
\end{table*}

\end{document}